\documentclass[a4paper,UKenglish,cleveref, autoref, thm-restate]{lipics-v2021}

\usepackage{amssymb,amsmath,amsthm}
\usepackage{mathrsfs}
\usepackage{xcolor}
\usepackage{hyperref}
\usepackage{bookmark}
\usepackage{xspace}
\usepackage{mathtools}
\usepackage{algorithm}
\usepackage[noend]{algpseudocode}
\usepackage[all]{xy}

\usepackage{tikz}
\usetikzlibrary{positioning}
\usetikzlibrary{arrows}
\usetikzlibrary{petri}
\usepackage{caption}

\usepackage{amsfonts}
\usepackage[T1]{fontenc}

\usepackage{apxproof}
\renewcommand{\appendixprelim}{\newpage}
\newtheoremrep{lemma}[theorem]{Lemma}
\newtheoremrep{claim}[theorem]{Claim}

\newcommand{\regplus}{\oplus}
\newcommand{\normalcounters}{H}
\newcommand{\atomcounters}{P}
\newcommand{\registers}{\mathcal{R}}
\newcommand{\Loc}{L}
\newcommand{\botval}{\bot\!\!\!\!\bot}
\newcommand{\powerset}[1]{\mathcal{P}(#1)}

\newcommand{\vect}[1]{\mathbf{#1}}
\newcommand{\vr}[1]{\vect{#1}}
\newcommand{\Conf}{\text{\sc Conf}}
\newcommand{\countersA}{\atomcounters {\times} \A}
\newcommand{\counters}{(\normalcounters \cup \atomcounters {\times} \A)}
\newcommand{\datavec}{\cg {\counters}}
\newcommand{\datavecN}{\cgN {\counters}}
\newcommand{\rxvaluation}[1]{#1 \to (\A \cup \{\bot\})}
\newcommand{\rvaluation}{\rxvaluation{\registers}}

\newcommand{\Autpar}[2]{\textnormal{Aut}_{#1}(#2)}
\newcommand{\Autparco}[2]{\text{Aut}_{#2\setminus#1}(#2)}
\newcommand{\Aut}[1]{\Autpar{}{#1}}
\newcommand{\cg}[1]{{#1}\to_\text{\tiny fin}\Z}
\newcommand{\cgN}[1]{{#1}\to_\text{\tiny fin}\N}
\newcommand{\vass}{\text{\sc vass}\xspace}
\newcommand{\dvass}{\text{\sc dvass}\xspace}
\newcommand{\dvas}{\text{\sc dvas}\xspace}
\newcommand{\pseudorunarrow}{\dashrightarrow}
\newcommand{\runarrow}{\longrightarrow}
\newcommand{\backpseudorunarrow}{\dashleftarrow}
\newcommand{\backrunarrow}{\longleftarrow}
\newcommand{\orbsize}[1]{||#1||}
\newcommand{\myTheta}{\mathbf{\Phi}}

\newcommand{\rank}[1]{\text{\sc Rank}(#1)}
\newcommand{\orbitof}[1]{\text{\sc Orbit}(#1)}

\newcommand{\csg}[1]{\text{\sc Graph}(#1)}

\newcommand{\supp}[1]{\text{\sc supp}(#1)}
\newcommand{\pair}[2]{\langle #1,#2 \rangle}
\newcommand{\wqorel}{\sqsubseteq}

\newcommand{\coverset}[1]{\text{\sc Cover}(#1)} 
\newcommand{\expspace}{{\sc ExpSpace}\xspace}
\newcommand{\Ack}{{\sc Ackermann}\xspace}
\newcommand{\singl}{*}
\newcommand{\cond}[2]{\text{\sc State-Eq}(#1, #2)}

\newcommand{\picunit}{3cm}
\newcommand{\halfunit}{1.1cm}

\newenvironment{slexample}
    {\begin{example}\rm}
    {\hfill $\triangleleft$ \end{example}}

\newcommand{\wqo}{\text{\sc wqo}\xspace}

\newcommand{\size}[1]{|#1|}

\newcommand{\A}{\mathbb{A}}

\newcommand{\Z}{\mathbb{Z}}

\newcommand{\N}{\mathbb{N}}
\newcommand{\Nom}{\N_\omega}

\newcommand{\V}{\mathcal{V}}

\newcommand{\down}[1]{#1\!\downarrow\,}

\newcommand{\set}[1]{\left\{#1\right\}}

\newcommand{\setof}[2]{\left\{#1 \, \middle\vert \, #2\right\}}

\newcommand{\setfromto}[2]{\set{#1\ldots #2}}
\newcommand{\setto}[1]{\setfromto 0 {#1}}

\newcommand{\prettyexists}[2]{\exists #1 : #2}

\newcommand{\defeq}{\stackrel{\text{\tiny def}}{=}}

\newcommand{\para}[1]{\subparagraph*{\rm \bf #1.}}

\newcommand{\probdef}[3]{
\begin{itemize} 
    \item[] \textsc{#1}
    \item[] \textit{Input:} \quad\, #2.
    \item[] \textit{Question:} #3?
\end{itemize}
}

\title{Bi-reachability in Petri nets with data}

\author{{\L}ukasz Kami{\'n}ski}{University of Warsaw, Poland}{}{https://orcid.org/0009-0004-1641-9049}
{Partially supported by NCN grant 2021/41/B/ST6/00535.}

\author{S{\l}awomir Lasota}{University of Warsaw, Poland}{}{https://orcid.org/0000-0001-8674-4470}
{Partially supported by the ERC grant INFSYS, agreement no. 950398.}

\authorrunning{{\L}.~Kami{\'n}ski and S.~Lasota} 

\Copyright{{\L}ukasz Kami{\'n}ski and S{\l}awomir Lasota} 



\acknowledgements{We are grateful to Piotrek Hofman for inspiring discussions.}

\nolinenumbers 

\EventEditors{}
\EventNoEds{2}
\EventLongTitle{}
\EventShortTitle{}
\EventAcronym{}
\EventYear{}
\EventDate{}
\EventLocation{}
\EventLogo{}
\SeriesVolume{42}
\ArticleNo{23}

\begin{document}

\maketitle

\begin{abstract}
We investigate Petri nets with data, an extension of plain Petri nets 
where tokens carry values from an infinite data domain, 
and executability of transitions is conditioned by equalities between data values.
We provide a decision procedure for the bi-reachability problem: given a Petri net and its two configurations,
we ask if each of the configurations is reachable from the other.
This pushes forward the decidability borderline, as
the bi-reachability problem subsumes the coverability problem (which is known to be decidable) and is subsumed
by the reachability problem (whose decidability status is unknown).
\end{abstract}

\section{Introduction}

We investigate the model of Petri nets with data, where tokens carry values from some fixed data domain, 
and executability of transitions is conditioned by relations between data values involved.
We study Petri nets with  \emph{equality} data~\cite{Lasota16,LNORW07,RF11}, i.e., a countable 
infinite data domain 
with  equality as the only  relation.
Other data domains have been also studied, for instance Petri nets with
\emph{ordered} data~\cite{LNORW07}, i.e., a countable infinite, densely and totally ordered data domain
(the model subsumes Petri nets with equality data).
One can also consider an abstract setting of Petri nets with an arbitrary fixed data domain \cite{Lasota16}.

As an illustrating example, consider a Petri net with equality data which has two places $p_1, p_2$ and two transitions $t_1, t_2$, as depicted in Fig.~\ref{fig:pn}.
Transition $t_1$ outputs two tokens with arbitrary but distinct data values onto place $p_1$. 
Transition $t_2$ inputs two tokens with the same data value, say $a$, one from $p_1$ and one from $p_2$, and outputs three tokens: two tokens with arbitrary but equal data values $b$, where $b\neq a$, one onto $p_1$ and the other onto $p_2$, plus one token with a data value $c \neq a$ onto $p_1$.
Note that transition $t_2$ does not specify whether $b = c$ or not, and therefore both 
options are allowed.
\begin{figure}[t] 
\centering
\begin{tikzpicture}
[auto,place/.style={circle,draw=black!50,fill=black!20,thick,inner sep=2mm},
transition/.style={rectangle,draw=blue!50,fill=blue!20,thick,rounded corners=0.5pt,inner sep=1mm}]
 
\node (P) [place,label=above:$p_1$] {}
  [children are tokens]
  child{node[token]{}}
  child{node[token,red]{}}
  child{node[token,red]{}};
\node (Q) [place,right=\picunit of P,label=above:$p_2$] {}
  [children are tokens]
  child{node[token,blue]{}}
  child{node[token]{}};
\node (T) [transition, left=\halfunit of P,label=above:\ ] {$t_1$}
  edge[post] node{$x_1, x_2$}(P);
\node (S) [transition, right=\halfunit of P,label=above:\ ] {$t_2$}
  edge[pre,bend right,above] node{$y_1$}(P)
  edge[pre,bend left,above] node{$y_2$}(Q)
  edge[post,bend left,below] node{$z_1, z_3$}(P)
  edge[post,bend right,below] node{$z_2$}(Q);
\node[rectangle,draw] at (-1.60 ,-1.25) (phi1) {$x_1 \neq x_2$};
\draw  [-] (T.south)      -- (phi1);
\node[rectangle,draw] at (2.15 ,-1.25) (phi2) {$z_1 = z_2 \neq y_1 = y_2 \neq z_3$};
\draw  [-] (S.south)      -- (phi2);
\end{tikzpicture}
\caption{A Petri net with equality data, with places $\set{p_1, p_2}$ and transitions $\set{t_1, t_2}$. 
The shown configuration engages 5 tokens, carrying 3 different data values, 
depicted through different colors.
\label{fig:pn}}
\end{figure}
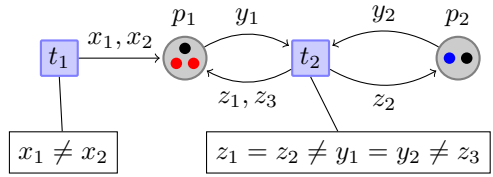

The most fundamental decision problem for Petri nets, the \emph{reachability} problem, asks,
given a net together with source and target configurations, if there is a run from source to target.
It is well known that the reachability problem is undecidable for Petri nets with ordered data \cite{LNORW07}, 
while the decidability status of this problem for equality data still remains an intriguing open question.
The same applies to two other major extensions of plain Petri nets, namely pushdown Petri nets 
\cite{LT17} and branching Petri nets \cite{Figueira17,Verma05}.
On the other hand, the \emph{coverability} problem (where we ask if there is a run from source to a configuration
that possibly extends target by some extra tokens) 
is decidable for both equality and ordered data \cite{Lasota16}.
As widely known, coverability easily reduces to reachability.
Furthermore, the reachability problem is decidable, also for equality and ordered data,
in the special case of \emph{reversible} Petri nets (where transitions are closed under reverse), 
as recently shown in \cite{GL24}.

In this paper we do a step towards decidability of reachability in Petri nets with equality data, and
study a relevant decision problem sandwiched between reachability and
the two latter decidable problems: the \emph{bi-reachability} problem 
(also called \emph{mutual reachability problem} or \emph{reversible reachability problem} \cite{L13}).
It asks, for a net and two its configurations, if each of the configurations is reachable from the other one.
In other words, the problem asks if two given configurations are in the same bi-reachability equivalence class.
Here are all know reductions, valid for Petri nets with either equality or ordered data,
as well as for plain Petri nets (without data):
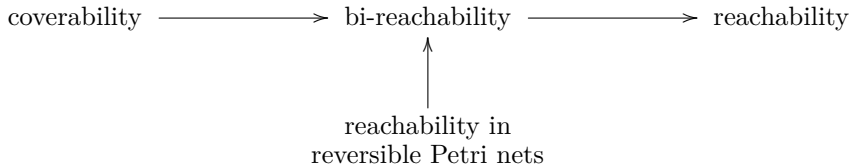
\[
\xymatrix{
\text{ coverability } \ar[rr] && \text{ bi-reachability } \ar[rr] && \text{ reachability } \\
&& \txt{reachability in\\reversible Petri nets}  \ar[u] &&
}
\]
As our main result we prove decidability of this problem for equality data domain.
This result pushes further the decidability border, subsuming decidability
of coverability, and of reachability in reversible Petri nets with equality data.
Our approach is specific to equality data, and thus we leave unresolved the status of bi-reachability 
in case of ordered data.

The decision procedure for bi-reachability 
is inspired by the classical decomposition approach used to decide reachability
in plain Petri nets \cite{Kosaraju82,Lambert92,Mayr81}.
There, it is often more convenient to work with 
\emph{vector addition systems with states} (\vass) instead of Petri nets \cite{Kosaraju82,Mayr81}.
Following this line, for technical convenience we prefer to work with the model of 
\emph{data \vass} (\dvass) \cite{HLLLST16} 
rather than with Petri nets.
In short, our approach consists of two ingredients.
First, we provide a sufficient condition for a \dvass to admit bi-reachability
(resembling $\Theta_1$ and $\Theta_2$ conditions of \cite{Kosaraju82,Mayr81}), which is effectively testable.
Second, in case the condition fails, we provide an effective way of reducing a \dvass to an equivalent
one, with respect to bi-reachability, which has smaller \emph{rank}.
As ranks are well founded, the reduction step guarantees correctness and termination.
Importantly, the decision procedure manipulates \dvass, and does not need to resort to manipulation of 
more general
structures (like generalised \vass of \cite{Kosaraju82,Mayr81}, or 
graph-transition sequences of \cite{Lambert92}, or witness graph sequences of \cite{LS15}, or
KLM sequences of \cite{LS19}).
This allows us to avoid similar generalisations in the data setting, and 
allows to keep the algorithm relatively simple.

Our work leaves two exciting open questions:
can one extend our approach to bi-reachability in case of ordered data,
or to reachability in case of equality data.
Clearly, if attempting to solve the latter problem, one unavoidably will be faced with some  
generalisation of the above-mentioned structures to the data setting.

\para{Related research}
Petri nets with equality data are a well established and widely studied model of concurrent systems, 
as data allow to model important aspects of such systems not captured by plain Petri nets, 
e.g. process identity \cite{Blondin23,Cervesato99}. 
The model can be also seen as a reinterpretation of the classical definition of Petri nets 
with a relaxed notion of finiteness, namely orbit-finiteness, where one allows for orbit-finite 
sets of places and transitions instead of just finite ones; this is along the lines of~\cite{BKL11full,BKLT13}. 
Similar net models have been proposed already in the early 80ies: 
high-level Petri nets~\cite{GL81} and colored Petri nets~\cite{J81}.
Since then, similar formalisms seem to have been rediscovered, for instance constraint multiset rewriting~\cite{CDLMS99,D02a,D02b}.

In plain Petri nets, bi-reachability is decidable as a consequence of decidability of reachability
\cite{Kosaraju82,Lambert92,Mayr81}.
Later, exact \expspace complexity was established in \cite{L13}.
In our setting, the problem is \Ack-hard due to \cite{LT17}.
In pushdown Petri nets, decidability of reachability
in the reversible subclass has been shown only recently \cite{GMPSZ22}, while decidability status of
bi-reachability is still open.
Indeed, it is known that reachability in pushdown Petri nets
with $d$ places reduces to coverability in pushdown Petri nets with $d+1$ places,
and the latter problem reduces to bi-reachability in pushdown Petri nets.
Hence, decidability of bi-reachability would imply decidability of reachability
in case of pushdown Petri nets.

\section{Preliminaries: orbit-finite sets and vectors}
In the sequel, let $\A$ denote a fixed countable infinite set of data values (called also \emph{atoms}).
By $\Aut \A$ we denote the set of all permutations of $\A$ (called also \emph{automorphisms}).
For a~subset $S \subseteq \A$ we define the subgroup 
$\Autpar S \A = \setof{\sigma \in \Aut \A}{\sigma(s) = s \mbox{ for all } s \in S}$.
Permutations in $\Autpar S \A$ we call \emph{$S$-automorphisms}.

\para{Orbit-finite sets}
In the following we study actions of the group $\Aut \A$ on different sets.
An action of $\Aut \A$ on a set $Z$ is a group homomorphism $\iota$ from $\Aut \A$ to functions $Z\to Z$.
We write $\sigma(z)$ instead of $\iota(\sigma)(z)$ for $\sigma \in \Aut \A$ and $z \in Z$.
In the sequel we always use the natural action of $\Aut \A$ that,
regardless of $Z$, renames atoms $a\in\A$ but leaves other elements intact.
Here are two specific examples of such action that will serve later as building blocks in the definition of
our model:
\begin{slexample} \label{ex:action}
Let $\bot\notin\A$.
For any finite sets $\Loc$ and $\registers$ of \emph{locations} and \emph{register names}, respectively,
the group $\Aut \A$ acts naturally on the set of \emph{states}
$Q = \Loc \times \left(\rvaluation\right)$,
namely given $\sigma \in \Aut \A$ and $q = (\ell, \nu) \in Q$,
we put 
\[
\sigma(q) := (\ell, \sigma(\nu)) \qquad \qquad \text{ where } 
\sigma(\nu)(r) = \begin{cases} 
\sigma(\nu (r)), & 
\text{ if } \nu(r) \in \A \\ 
\bot & \text{ if }\nu (r) = \bot.
\end{cases}
\]
Furthermore, for any finite sets $\normalcounters$, $\atomcounters$ of plain \emph{places} and atom \emph{places},
respectively,
$\Aut \A$ acts naturally on functions $\normalcounters \cup \atomcounters\times \A\to \Z$, namely
given $\sigma\in\Aut \A$ and $\vr v : \normalcounters \cup \atomcounters\times\A\to\Z$ we put
$\sigma(\vr v)(h) := \vr v (h)$ for $h \in \normalcounters$, and
$\sigma(\vr v)(p, \sigma(a)) := \vr v(p, a)$ for $p \in \atomcounters$.
\end{slexample}

Roughly speaking, a~set is \emph{orbit-finite}
if it has a~finite number of elements up to automorphisms of atoms.
We define 
the~\emph{orbit} of an~element $z \in Z$:
\begin{align*}
    \orbitof{z} := \setof{\sigma(z)}{\sigma\in {\Aut \A}}.
\end{align*}
As different orbits are necessarily disjoint, $Z$ partitions uniquely into orbits.
A subset $X\subseteq Z$ is \emph{orbit-finite} if it is a~finite union
of orbits.
Clearly all orbits, and hence also all finite unions thereof, are
closed under the action of $\Aut \A$.
Orbit-finite sets are closed under finite unions and products \cite[Lem.~3.24]{atombook}.

\begin{slexample} \label{ex:of}
We continue Example \ref{ex:action}.
The whole set $Q = \Loc \times \left(\rvaluation\right)$ is orbit-finite, since the orbit of a state 
$q=(\ell, \nu) \in Q$ is determined by its location $\ell$, the inverse image of $\bot$, namely 
$\setof{r\in \registers}{\nu(r)=\bot}$,
and the \emph{equality type} of $\nu$, namely the set $\setof{(r,r')\in \registers^2}{\nu(r)=\nu(r')\neq \bot}$.
Indeed, for every two states $q=(\ell, \nu)$ and $q' = (\ell', \nu')$ such that $\ell =\ell'$,
and $\nu$ and $\nu'$ have the same inverse image of $\bot$ and the same equality type, 
there is an automorphism $\sigma\in\Aut \A$ such that $\sigma(q)=q'$.

On the other hand, the function space $(\normalcounters \cup \atomcounters \times\A)\to\Z$ is not orbit-infinite.
\end{slexample}

\para{Vectors}
Given a set $X$, by $\cg X$ we denote the commutative group freely generated by $X$,
and the group operation we denote by $\oplus$.
We write $\vr v \ominus \vr w$ instead of $\vr v \oplus \vr w^{-1}$. 
Equivalently, $\cg X$ can be identified with the set of all functions 
$\vect v : X\to \Z$ which map almost all elements of $X$ to 0, i.e., those functions
where the set $\setof{x\in X}{\vr v(x)\neq 0}$ is finite.
Elements of $\cg X$ we call \emph{$X$-vectors}, or simply \emph{vectors} if  
the generating set $X$ is clear from the context.
The zero vector we denote by $\vr 0$, irrespectively of $X$, and a single-generator vector $x\in X$ we
denote by $\vr 1_x$.
Seen as a function $X\to\Z$, the vector $\vr 1_x$ maps $x$ to 1 and all other elements of $X$ to 0.
When $X$ is finite, we call $X$-vectors finite as well.
Nonnegative vectors, denoted $\cgN X$, 
are elements of the commutative monoid freely generated by $X$, or functions
$\vect v : X\to \N$ which map almost all elements of $X$ to 0 or, equivalently, finite multisets of
elements of $X$.
We write $\oplus W$ to denote the sum of a finite set  $W$ of vectors.

In the sequel the generating set $X$ is most often of the form $\normalcounters \cup \atomcounters {\times}\A$ 
for some finite sets $\atomcounters$, $\normalcounters$. 
Clearly, $\datavec$ is isomorphic to $(\cg \normalcounters) \times (\cg{\atomcounters{\times}\A})$, as every vector $\vr v : \datavec$
decomposes uniquely as the sum $\vr v = \vr u \oplus \vr w$, where
$\vr u : \cg \normalcounters$,  $\vr w : \cg{\atomcounters{\times}\A}$.
Given $\vr v : \datavec$, we define its support $\supp {\vr v} \subseteq \A$, 
as the (necessarily finite) set of those atoms 
which are sent by $\vr v$ to a nonzero value:
    \[\supp {\vr v} \ := \ \setof{a\in \A}{\prettyexists{p\in P}{\vr v(p,a)\neq 0}}.\]
Intuitively, $\supp {\vr v}$ contains those atoms that 'appear' in $\vr v$.
We observe that $\sigma(\vr v) = \vr v$ as long as $\sigma(a)=a$ for all $a\in\supp {\vr v}$.
The natural action of $\Aut \A$ given in Example \ref{ex:action} restricts to the set of vectors 
$\datavec$, which is still not orbit-infinite.

\begin{slexample} \label{ex:vect}
Transitions $t_1, t_2$ of Petri net in Figure \ref{fig:pn} are semantically 
orbit-finite sets of $(\atomcounters {\times}\A)$-vectors, where $P=\set{p_1,p_2}$
(i.e., $\normalcounters = \emptyset$). 
Indeed, the effect of firing $t_1$ amounts to adding two arbitrary but different atoms $a\neq b$ to place $p_1$,
i.e., is described by an $X$-vector 
\[
\vr v_1 \ = \ (p_1, a) \oplus (p_1, b) \qquad\qquad (a\neq b).
\]
As the choice of atoms $a\neq b$ is arbitrary, all possible effects of firing the transition span one orbit of vectors:
$V_1 = \orbitof{\vr v_1}$.
The effect of firing $t_2$ amounts to removing some arbitrary atom $a$ from both $p_1$ and $p_2$, and
adding two further atoms $b,c$ not equal to $a$: one of them is added to both $p_1$ and $p_2$,
while the other one only to $p_1$.
As it is not specified  whether $b = c$ or not, we describe $t_2$ by two $X$-vectors:
\begin{align} \label{eq:v1v2}
& \vr v_2 \ = \ (p_1, c) \oplus (p_1, b) \oplus (p_2, b) \ominus (p_1, a) \ominus (p_2, a) \qquad\qquad
(a\neq b \neq c \neq a)\\
& \vr v'_2 \ = \ (p_1, b) \oplus (p_1, b) \oplus (p_2, b) \ominus (p_1, a) \ominus (p_2, a) \qquad\qquad
(a\neq b). 
\end{align}
As before, the choice of atoms is arbitrary, and hence all possible effects of firing $t_2$ span
the union of two orbits of vectors:
$V_2 = \orbitof{\vr v_2} \cup \orbitof{\vr v'_2}$.
Intuitively, different orbits in $T_2$ correspond to different \emph{equality types} of a tuple of atoms $(a,b,c)$:
one defined by inequalities $a\neq b\neq c\neq a$, and another defined by $a\neq b = c$.
This example illustrates a transformation of Petri nets to data \vass, the model we work with in this paper.%
\footnote{
As in a transformation from plain Petri nets to \vass,
in case of Petri net with \emph{tight loops}, i.e., transitions that simultaneously input and output the same atom from/to
the same place, we would have to split every such transition into input and output part.
}
\end{slexample}

\para{Multiset sum problem}
The following core decision problem, parametrised by an orbit-finite set $X$, will be useful later:

\probdef{Multiset Sum}
{an orbit-finite set $M$ of $X$-vectors, and an $X$-vector $\vr b$}
{is $\vr b$ equal to the sum of a finite multiset of vectors from $M$}

\noindent
In other words, we ask if $\vr b$ is a nonnegative integer linear combination of vectors from $M$.
We assume that $M$ is represented by a finite set of representatives, one per orbit.

\begin{lemma}[{\cite[Thm.~17]{GHL23}}] \label{lem:HLT17}
{\sc Multiset Sum} is decidable.
\end{lemma}

\section{Data vector addition systems with states}

Classical Petri nets are equivalent, with respect to most decision problems, to vector addition systems
(\vass).
Likewise, we introduce here a formalism equivalent to Petri nets with data, called data vector addition systems
with states (\dvass).
It is an extension of (stateless) data vector addition systems (\dvas) studied in \cite{HLLLST16}.

\para{Data \vass}
A data \vass (\dvass) $\V = (\Loc, \registers, \normalcounters, \atomcounters, T)$ consists of 
pairwise disjoint finite sets of locations $\Loc$, register names $\registers$, 
plain places $\normalcounters$, atom places $\atomcounters$, and an~\emph{orbit-finite}
set 
\[ 
T \subseteq Q \times \left(\datavec \right) \times Q
\] 
of \emph{transitions},
where $Q = \Loc \times \left(\rvaluation\right)$ is the~set of \emph{states}.
The set $T$ is thus assumed to be a finite union of orbits, under
the natural action of $\Aut \A$ on transitions that extends
the action on vectors and states given in Example \ref{ex:action}:
for $t=(q, \vr v, q') \in T$ we put $\sigma(t) := (\sigma(q), \sigma(\vr v), \sigma(q'))$.
In particular, $T$ is closed under the action of $\Aut \A$.
Given a~state $q=(\ell, \nu) \in Q$, the function
$\nu$ is called \emph{register valuation}.
Intuitively, $\nu(r)=a$ means that register $r$ contains atom $a$,
while $\nu(r)=\bot$ means that $r$ is empty.
The vector $\vr v$ is called the \emph{effect} of transition $(q, \vr v, q')$.

The model of (plain) \vass corresponds to the special case where $\registers =\emptyset$ and $\atomcounters = \emptyset$, i.e., \dvass without registers and atom places.
In this case the set $T \subseteq Q \times (\cg \normalcounters) \times Q$, being orbit-finite, is necessarily finite.
The model of \dvas corresponds to the special case when $\Loc=\set{\singl}$ is a singleton and $\registers = \emptyset$
and $\normalcounters=\emptyset$, i.e., \dvass without locations, registers and plain places.

\smallskip

A~\emph{pseudo-configuration} of $\V$ is a pair $(q, \vr v) \in Q \times (\datavec)$,
written $q(\vr v)$.
A~\emph{pseudo-run}
from $q_1(\vect{v}_0)$ to $q_k(\vect{v}_k)$ is a~sequence of pseudo-configurations
$\pi = q_0(\vect{v}_0)\, q_1(\vect{v}_1)\, \dots\, q_k(\vect{v}_k)$ such that 
$t_i = (q_{i-1}, \vect{v}_i-\vect{v}_{i-1}, q_i) \in T$ for every $i = 1, \dots, k$. 
We say that the pseudo-run $\pi$ \emph{uses} the transitions $t_1, \ldots, t_k \in T$.
The~support of a~state $q = (\ell, \nu)$ is the set of all atoms used in registers, i.e. 
$
\supp q = \nu (\registers) \cap \A.
$
The support of a transition $t=(q,\vr v, q')$ is $\supp t = \supp q \cup \supp {\vr v} \cup \supp {q'}$.
We also define the support of a pseudo-run as the union of supports of all its pseudo-configurations:
$\supp \pi = \supp {q_0} \cup \supp {\vr v_0} \cup \supp {q_1} \cup \ldots \cup \supp {q_k} \cup \supp {\vr v_k}$.
We again extend the action of $\Aut \A$, this time to pseudo-runs, in an expected way:
\[
\sigma \big(q_0(\vect{v}_0)\, q_1(\vect{v}_1)\, \dots\, q_k(\vect{v}_k)\big) \ := \ 
\sigma(q_0)(\sigma(\vect{v}_0))\ \sigma(q_1)(\sigma(\vect{v}_1))\ \dots\ \sigma(q_k)(\sigma(\vect{v}_k)).
\]
The set of pseudo-runs is closed under the action of $\Aut \A$.

\smallskip

\emph{Configurations} are those
pseudo-configurations $q(\vr v)$ where the vector $\vr v$ is nonnegative, i.e., $\vr v : \datavecN$.
Let $\Conf = Q\times\big(\datavecN\big)$ denote the set of all configurations.
A~\emph{run} is a~pseudo-run where every pseudo-configuration $q_i(\vr v_i)$ is actually a~configuration.
We write $q(\vr v) \pseudorunarrow q'(\vr v')$
(resp.~$q(\vr v) \runarrow q'(\vr v')$) if there is a~pseudo-run (resp.~a run)
from $q(\vr v)$ to $q'(\vr v')$.

\begin{slexample} \label{ex:dvass}
Continuing Example \ref{ex:vect}, Petri net in Figure \ref{fig:pn} is equivalent to a \dvas
$\V = (\set{\singl}, \emptyset, \emptyset, \set{p_1, p_2}, T)$, whose transitions are
(as $\registers=\emptyset$, we omit register valuations)
\[
T \ = \ \set{\singl}\times(V_1 \cup V_2)\times\set{\singl}.
\]
The initial configuration shown in Figure \ref{fig:pn} is $\singl(\vr v)$,
where 
$\vr v = (p_1, a) \oplus (p_1, c) \oplus (p_1, c) \oplus (p_2,a) \oplus (p_2, b)$ for some distinct atoms
$a, b, c \in \A$.
In order to illustrate \dvass, we drop the first inequality in the constraint on $t_2$, 
and consider the relaxed constraint
$z_1 = z_2 \land y_1 = y_2 \neq z_3$ instead.
This adds a third orbit of possible effects of firing $t_2$, when the atom $b$ added to places
$p_1$ and $p_2$ is the same as the atom $a$ removed 
(c.f.~\eqref{eq:v1v2} in Example \ref{ex:vect}):
\begin{align*}
& \vr v''_2 \ = \ (p_1, c) \oplus (p_1, a) \oplus (p_2, a) \ominus (p_1, a) \ominus (p_2, a) \qquad\qquad
(a\neq c). 
\end{align*}
The modified Petri net is equivalent to a
\dvass $\V' = (\Loc, \emptyset, \emptyset, \atomcounters, T')$ 
with two locations $\Loc=\set{\ell,\ell'}$,
still no registers, 
a~larger set of atom places $P = \set{p_1, p_2, \overline p}$, and transitions:
\[
T' \ = \ \set{\ell}\times(V_1 \cup V_2)\times\set{\ell}
\ \ \cup \ \ \set{\ell}\times\orbitof{\vr w}\times\set{\ell'}
\ \ \cup \ \ \set{\ell'}\times\orbitof{\vr w'}\times\set{\ell},
\]
where vectors $\vr w$, $\vr w'$ are splitting $\vr v''_2$ into input and output part, using an
auxiliary place $\overline p$ to temporarily store atom $a$:
\begin{align} \label{eq:w}
\vr w \ = \ (\overline p, a)  \ominus (p_1, a) \ominus (p_2, a) \qquad
\vr w' \ = \ (p_1, c) \oplus (p_1, a) \oplus (p_2, a) \ominus (\overline p, a) \quad
(a\neq c). 
\end{align}
Transitions in $T'$ corresponding to $V_1\cup V_2$ go from $\ell$ to $\ell$, while the other transitions go
from $\ell$ to $\ell'$, or from $\ell'$ to $\ell$.
The initial configuration of $\V'$ is $\ell(\vr v)$.
Instead of place $\overline p$ one could also use a single register $\registers=\set{r}$, 
and transitions of the form
\[
  ((\ell, \bot), \ \ominus (p_1, a) \ominus (p_2, a), \ (\ell', a)) \qquad 
    ((\ell', a), \ (p_1, c) \oplus (p_1, a) \oplus (p_2, a), \ (\ell, \bot))
    \qquad
(a\neq c),
\]
to the same effect as in \eqref{eq:w}. The initial configuration would be
then $q(\vr v)$, where $q=(\ell, \bot)$.
\end{slexample}

\begin{remark} \label{rem:dvass2dvas}
Our model of \dvass syntactically extends \dvas by locations, registers and plain places.
The extended model is convenient for our decidability argument, while being equivalent to \dvas
with respect to most of decision problems.
Indeed, a \dvass may be transformed into an essentially equivalent \dvas in three steps
(as in the proof of Lemma \ref{lem:cover}):
\begin{align} \label{eq:dvass2dvas}
\xymatrix{
\txt{locations\ } \ar_2[r]  & \txt{\ plain places} \ar^3[d] \\
\txt{registers\ } \ar^1[r] \ar^1[u] & \txt{\ atom places}
}
\end{align}
First, we eliminate registers using locations and atom places, then we encode locations into plain places,
and finally we encode plain places into atom ones.
\end{remark}

\para{State graph}
We define the \emph{state graph} $\csg T = (Q, E)$ of a \dvass $\V$.
Its nodes are states $Q$, and its edges $E\subseteq Q\times Q$ are pairs of states related by some
transition of $\V$:  
\[
E \ = \ \setof{(q,q')\in Q^2}{(q, \vr v, q')\in T \text{ for some vector } \vr v : \datavec}.
\]
When $\registers$ is non-empty, the sets of nodes and edges of $\csg T$ are infinite but orbit-finite.

\para{Bi-reachability problem}
We say that a configuration $q'(\vr v')$ is \emph{reachable} from $q(\vr v)$ if 
there is a run $q(\vect{v}) \runarrow q'(\vect{v}')$.
Two configurations $q(\vr v)$, $q'(\vr v')$ are \emph{bi-reachable} if each of them is reachable from the other:
$q(\vect{v}) \runarrow q'(\vect{v}')$ and $q(\vect{v}) \backrunarrow q'(\vect{v}')$.

\probdef{\dvass bi-reachability}
{a \dvass $(\Loc, \registers, \normalcounters, \atomcounters, T)$ and two configurations, $q(\vr v)$ and $q'(\vr v')$
}
{are $q(\vr v)$, $q'(\vr v')$ bi-reachable}

\noindent
As before, we assume that the orbit-finite set $T$ of transitions is represented by a finite set
of representatives, one per orbit.
As our main result we prove:
\begin{theorem} \label{thm:bireach}
\dvass bi-reachability problem is decidable.
\end{theorem}
Since our model of \dvass includes plain places, we can assume w.l.o.g.~a convenient form of source 
and target configuration that consists, essentially, of just a location.
Let $\botval$ denote the empty register valuation: $\botval(r)=\bot$ for every $r\in\registers$. 
\begin{lemmarep} \label{lem:wlog}
    In {\sc \dvass bi-reachability} problem we may assume,
    w.l.o.g., that $q=(\ell,\botval)$,
    $q'=(\ell',\botval)$,
    and $\vect{v}=\vect v' = \vect{0}$.
\end{lemmarep}
%
%
\begin{proof}
    Consider a \dvass $\V = (\Loc, \registers, \atomcounters, \normalcounters, T)$ 
    and two configurations $q(\vr v), q'(\vr v')$, where $q = (\ell, \nu)$, $q'=(\ell',\nu')$.
    We proceed in three steps, as shown in the diagram (cf.~diagram \eqref{eq:dvass2dvas} in 
    Remark \ref{rem:dvass2dvas}):
\begin{align*} 
\xymatrix{
\txt{locations\ }  & \txt{\ plain places} \ar^3[l]  \\
\txt{registers\ } \ar^1[r] \ar^1[u] & \txt{\ atom places} \ar^2[u] 
}
\end{align*}
    As the first step we redo the first step of the proof of Lemma \ref{lem:cover} which yields
    a \dvass $\V_1 = (\Loc_1, \emptyset, \normalcounters, \atomcounters_1, T_1)$ without registers
    (which implies $q=(\ell, \botval)$ and $q'=(\ell', \botval)$).
    We choose initial and final location 
    \[ 
    \overline \ell \ := \ (\ell, \nu^{-1}(\bot)) \qquad \qquad \overline\ell'  \ := \ (\ell', (\nu')^{-1}(\bot)) \in \Loc_1
    \] 
    and, identifying a register valuation $\nu$ with the vector $\oplus \setof{(r,a)}{\nu(r)=a\neq\bot}$,
    we choose
    initial and final vector $\cgN{\normalcounters\cup\atomcounters_1{\times}\A}$,
    \[
    \overline {\vr v} \ := \ \vr v\oplus\nu \qquad \qquad
    \overline {\vr v}' \ := \ \vr v'\oplus\nu',
    \]
    and claim that the reachability is preserved (we omit registers, and write
    e.g.~$\overline\ell(\overline {\vr v})$):
    \begin{claim}
    The configurations $q(\vr v)$, $q'(\vr v')$ are bi-reachable in $\V$
    if and only if $\overline\ell(\overline {\vr v})$, $\overline\ell'(\overline {\vr v}')$ are bi-reachable in $\V_1$.
    \end{claim}
 
    Second, consider a register-less \dvass $\V_1 = (\Loc_1, \emptyset, \normalcounters_1, \atomcounters_1, T_1)$ 
    and two configurations $\ell(\vr v) = q(\vr u \oplus \vr w)$ and $\ell'(\vr v') = q'(\vr u' \oplus \vr w')$,
    where 
    $\vr u, \vr u' : \cgN {\normalcounters_1}$ and $\vr w, \vr w' : \cgN{\atomcounters_1{\times}\A}$.
    We argue that w.l.o.g.~we can assume
    $\vr w = \vr w' = \vr 0$.
    Let $S = \supp {\vr v} \cup \supp {\vr v'}$ 
    be the set of those atoms which appear in $\vr w$ or $\vr w'$.
    Intuitively, we move the set $\atomcounters_1\times S$ to plain places. 
    We take $\normalcounters_2 = \normalcounters_1 \cup \atomcounters_1{\times} S$ as plain places,
    and consider atoms $\A' = \A\setminus S$ instead of $\A$. Clearly,
    \[
    \normalcounters_1\cup (\atomcounters_1{\times}\A) \ = \ \normalcounters_2 \cup (\atomcounters_1{\times\A'})
    \]
    and therefore we may take the same transitions $T_1$ as transitions of the new \dvass
    $\V_2 = (\Loc_1, \emptyset, \normalcounters_2, \atomcounters_1, T_1)$.
    As $S$ is finite, the set $T_1$ is still orbit-finite with respect to $\Aut{\A'}$.
    \begin{claim}
    The configurations $\ell({\vr v})$, $\ell'({\vr v}')$ are bi-reachable in $\V_1$
    if and only if $\ell({\vr v})$, $\ell'({\vr v}')$ are bi-reachable in $\V_2$.
    \end{claim}
    
    In the last third step, consider a \dvass $\V = (\Loc_2, \emptyset, \normalcounters_2, \atomcounters_2, T_2)$ without registers
        and two configurations
    $\ell(\vr u)$ and $\ell'(\vr u')$,
    where $\vr u, \vr u' : \cgN {\normalcounters_2}$.
    We 
    eliminate the initial and final values $\vr u$, $\vr u'$ on plain places
    in a classical way, by introducing new initial and final locations $\overline \ell$, $\overline \ell'$
    and adding to $T_2$ the following four transitions, 
    and thus defining $\V_3 = (\Loc_2\cup\set{\overline {\ell}, \overline {\ell}'}, \emptyset, \normalcounters_2, \atomcounters_2, T_3)$:
    \begin{align} \label{eq:4tran}
    (\overline \ell, \vr u, \ell) \qquad   (\ell', - \vr u', \ \overline \ell') \qquad     (\ell, - \vr u,\  \overline \ell) \qquad
    (\overline \ell', \vr u', \ell').
    \end{align}
    \begin{claim}
    The configurations $\ell(\vr u)$,  $\ell'(\vr u')$ are bi-reachable in $\V_2$ 
    if and only if $\overline \ell(\vr 0)$,  $\overline \ell'(\vr 0)$ are bi-reachable in $\V_3$.
    \end{claim}
    Indeed, a run $\ell(\vr u) \runarrow \ell'(\vr u')$ in $\V_2$ extended with
    the first two transitions in \eqref{eq:4tran} yields a run
    $\overline \ell(\vr 0) \runarrow \overline \ell'(\vr 0)$ in $\V_3$, and likewise 
    a run $\ell'(\vr u') \runarrow \ell(\vr u)$ in $\V_2$ extended with
    the last two transitions in \eqref{eq:4tran} yields a run
    $\overline \ell'(\vr 0) \runarrow \overline \ell(\vr 0)$ in $\V_3$.
    Conversely, consider a run $\pi : \overline \ell(\vr 0) \runarrow \overline \ell'(\vr 0)$ in $\V_3$.
    It necessarily starts with the first transition in \eqref{eq:4tran}, and ends with the second one.
    If transitions \eqref{eq:4tran} are used elsewhere in $\pi$, they are necessarily used in pairs,
    namely the second one followed immediately by the fourth one, or the third  one is followed immediately
    by the first one. Effects of each such pair cancel out, and thus each pair can be safely removed from $\pi$.
    Finally, removing the first and the last transition makes $\pi$ into a run
    $\ell(\vr u) \runarrow \ell'(\vr u')$ in $\V_2$, as required.
    Likewise we transform a run $\overline \ell'(\vr 0) \runarrow \overline \ell(\vr 0)$ in $\V_3$.
\end{proof}

\section{Toolset}

Our decision procedure relies on a number of existing tools.
One of them is solvability of {\sc Multiset sum} (Lemma \ref{lem:HLT17}).
Here we introduce two further tools:
a sufficient condition for \vass reachability, and computability of 
coverability sets in \dvas.

\para{Sufficient condition for \vass reachability}

We recall a condition that guarantees existence of a run in a \vass.
It is a simplification of the classical condition of \cite{Kosaraju82,Lambert92,Mayr81}
which guarantees existence of a run in a \emph{generalised} \vass.
Consider a \vass $\V$ with plain places $\normalcounters$.
For $\vr v : \cgN {\normalcounters}$ 
we write $\vr v \gg \vr 0$ to mean that
for every $h \in \normalcounters$ we have $\vr v(h) > 0$.

\begin{itemize}
    \item$\Theta_1$: 
    For every $m\in \N$, there is a pseudo-run $q(\vect{0}) \pseudorunarrow q'(\vect{0})$
    using every transition at least $m$ times.
    \item$\Theta_2$: 
    For some vectors $\vr \Delta, \vr \Delta' \gg \vr 0$, there are runs:
        $q(\vect{0}) \runarrow q(\vr \Delta)$ and
        $q'(\vr \Delta') \runarrow q'(\vect{0})$.
\end{itemize}

\begin{lemma}[Thm.~2 in \cite{Kosaraju82}, Prop.~1 in \cite{Las3steps}] 
\label{lem:suffvass}
For every \vass,
$\Theta_1 \land \Theta_2$ implies $q(\vect{0}) \runarrow q'(\vect{0})$.
\end{lemma}

\para{Coverability sets in \dvass}

Let $\V = (\Loc, \registers, \normalcounters, \atomcounters, T)$ be a \dvass, 
and let $X = \normalcounters \cup(\atomcounters {\times}\A)$.
We define the pointwise order on nonnegative vectors $\cgN X$: $\vr v \leq \vr v'$ if and only
if for every $x \in X$ we have $\vr v(x) \leq \vr v'(x)$.
We define a quasi-order by relaxation of $\leq$, up to automorphisms:
$\vr v \wqorel \vr v'$ if $\sigma(\vr v) \leq \vr v'$ for some $\sigma \in \Aut \A$.
We extend the relation $\wqorel$ to configurations:
for states $q,q'$
and vectors $\vr v, \vr v' \in \cgN X$ we put
$q(\vr v) \wqorel q'(\vr v)$ if
$\sigma(q) = q'$ and $\sigma(\vr v) \leq \vr v'$ for some $\sigma \in \Aut \A$.
\begin{lemmarep} \label{lem:wqo}
$\wqorel$ is a \wqo on configurations.
\end{lemmarep}
%
%
\begin{proof}
Recall the sets of states $Q = \Loc \times (\registers\to(\A\cup\set{\bot}))$ and configurations
$\Conf = Q\times\big(\datavecN\big)$.
The quasi-order $\wqorel$ is a \wqo on
the set of nonnegative vectors $\cgN{P{\times}\A}$, 
as it is quasi-order-isomorphic to $M(\cg \atomcounters)$,
the set of finite multisets of finite vectors from $\cg \atomcounters$, ordered by multiset inclusion.
Furthermore, $\wqorel$ is a \wqo on $\datavecN$, 
as it is quasi-order-isomorphic to the Cartesian product $(\cgN{\normalcounters}) \times (\cgN \countersA)$
of two \wqo's, and Cartesian product preserves \wqo.

Every register valuation $\nu : \rvaluation$ may be seen as an $\registers'$-vector,
where $\registers' = \nu^{-1}(\A)$ is the set of non-empty registers, namely
$\widehat \nu = \oplus \setof{(r,a)}{\nu(r)=a\neq\bot} : \cgN {\registers'}$.
We use this fact to argue that $\wqorel$ is a \wqo on $Y=(\rvaluation) \times (\datavecN)$.
Indeed, we split this set into $2^{\size \registers}$ subsets, determined by non-empty registers, i.e.,
for every subset $\registers'\subseteq \registers$ we consider a subset
\[
C_{\registers'} \ := \ \setof{(\nu, \vr v)}{\nu^{-1}(\A)=\registers'} \ \subseteq \ Y.
\]
For every fixed $\registers'$, the set $C_{\registers'}$ is essentially a subset of $\cgN{\normalcounters \cup (\atomcounters\cup \registers')\times\A}$, 
due to the bijection $(\nu, \vr v) \mapsto  \widehat \nu \oplus \vr v$,
containing
those vectors which use exactly one generator from $\registers'\times\A$.
Therefore $C_{\registers'}$ is a \wqo.
In consequence, $Y$  is a \wqo too, as
finite sums preserve \wqo.

Finally, the set $\Conf = \Loc \times Y$ is a \wqo, as Cartesian product of the finite set $\Loc$ and
a~\wqo.
\end{proof}

The \emph{coverability set} of a configuration $q(\vr v)$ is defined as the downward closure, 
with respect to $\wqorel$, of the reachability set:
\[
\coverset  {q(\vr v)} \ = \ \setof{\overline s(\overline{\vr w}) \in \Conf}{\prettyexists{s(\vr w) \in \Conf}{\overline s(\overline{\vr w}) \wqorel s(\vr w) \ \wedge \ q(\vr v) \runarrow s(\vr w)}}.
\]
It is known that the coverability set is representable by a finite union of \emph{ideals} (downward closed directed sets) \cite{FG09,FG12}.
Let's complete $\N$ with a top element,
$\Nom \defeq \N\cup\set{\omega}$, which is larger than all numbers: $n< \omega$ for all $n\in\N$.
We consider pairs $(q, f) \in Q \times (X\to \Nom)$, written $q(f)$, and called \emph{$\omega$-configurations}.
Each such pair determines a set of configurations
(we extend $\wqorel$ to all  $\omega$-configurations in the expected way):
\[
\down {q(f)} \ := \ \setof{s(\vr v) \in \Conf}{s(\vr v) \wqorel q(f)},
\]
which is downward closed (whenever $s(\vr v) \in\down {q(f)}$ and $s'(\vr v')\wqorel s(\vr v)$ then 
$s'(\vr v')\in\down {q(f)}$) 
and directed (for every two $s(\vr v), s'(\vr v') \in \down {q(f)}$ there is 
$\overline{s} (\overline{\vr v}) \in\down{q(f)}$ such that $s(\vr v) \wqorel \overline{s}(\overline{\vr v})$ 
and $s'(\vr v') \wqorel \overline{s}(\overline{\vr v})$). 
The set $\down {q(f)}$ is thus an \emph{ideal}.
We call an $\omega$-configuration $q(f)$ \emph{simple} if for every $p\in \atomcounters$, either $f(p, a) = 0$ for 
\emph{almost all} $a\in\A$
(i.e., for all $a\in\A$ except finitely many),
or $f(a, p) = \omega$ for almost all $a\in\A$.
Simple $\omega$-configurations are thus finitely representable.
Ideals determined by simple $\omega$-configurations we call simple too.

\begin{slexample}
In the \dvass $\V'$ in Example \ref{ex:dvass},
$
\coverset  {\ell(\vr v)} = \down {\ell(f)} \cup \ \down {\ell(g)}  \cup \  \down {\ell'(f')} \cup \  \down {\ell'(g')},\
$
where $f(p_1, c) = g(p_1, c) = f'(p_1, c) = g'(p_1, c) = \omega$ for every $c\in\A$,
\begin{align*}
& f(p_2, a) = f(p_2, b) = 1 &&  g(p_2, a) = 2 \\ 
& f'(p_2, a) = f'(\overline p, b) = 1 &&  g'(p_2, a) = g'(\overline p, a) = 1
\end{align*}
for some $a\neq b\in\A$,
and all other arguments are mapped by $f$, $g$, $f'$ and $g'$ to 0.
Indeed, due to transition $t_1$, place $p_1$ can be filled up with arbitrary many tokens with any atoms.
On the other hand place $p_2$  has two tokens in the initial
configuration $\ell(\vr v)$, and hence will invariantly have, in location $\ell$, 
two tokens whose atoms may be equal or not.
Furthermore, in location $\ell'$, places $p_2$ and $\overline p$ have always one token each, with atoms equal or not.
\end{slexample}
Simple $\omega$-configurations provide finite representations of simple ideals.
Relying on the result of \cite{HLLLST16}, the coverability set in a \dvas is a union of a finite set of simple ideals, 
which is computable.
We lift this result to the model of \dvass:
\begin{lemma} \label{lem:cover}
Given a \dvass and its configuration $q(\vr v)$, 
one can compute a finite set of simple $\omega$-configurations
$\set{s_1(f_1), \ldots, s_n(f_n)}$ such that 
$\coverset {q(\vr v)} = \down {s_1(f_1)} \cup \ldots \cup \down{s_n(f_n)}$.
\end{lemma}
\begin{proof}
Let $\V = (\Loc, \registers, \normalcounters, \atomcounters, T)$  be a \dvass, let $q(\vr v)$ be its configuration, where
$q = (\ell, \eta)$.
Theorem 3.5 in \cite{HLLLST16} proves the claim in the special case of \dvas.
We reduce \dvass to \dvas in three steps, as shown in the diagram \eqref{eq:dvass2dvas} in Remark
\ref{rem:dvass2dvas}.

As the first step we get rid of registers 
by considering them as additional atom places that store at most one token, while
keeping track, in locations, of the set of currently empty registers.
We set $\atomcounters_1 := \atomcounters \cup \registers\cup\overline\registers$, where 
$\overline\registers = \setof{\overline r}{r\in\registers}$ is distinct a copy of $\registers$,
and 
$\Loc_1 = (\Loc \cup \overline \Loc) \times \powerset{\registers}$, where 
$\overline \Loc = \setof{\overline \ell}{\ell\in\Loc}$, and
define the new set of transitions $T_1$ by transforming transitions from $T$ 
as follows.
In the construction we identify a register valuation $\mu$ with a vector
$\mu = \oplus \setof{(r,a)}{\mu(r)=a\neq \bot}$,
or with a vector $\overline\mu = \oplus \setof{(\overline r,a)}{\mu(r)=a\neq \bot}$.
Every transition $t=((\ell, \mu), \vr v, (\ell', \mu')) \in T$ gives rise to a transition in $T_1$
\[
\Big((\ell, \mu^{-1}(\bot)), \ \vr v \ominus \mu \oplus \overline \mu', \ (\overline\ell', (\mu')^{-1}(\bot))\Big)
\]
that starts in location $(\ell, \mu^{-1}(\bot))$, ends in a location $(\overline\ell', (\mu')^{-1}(\bot))$ and whose
effect is $\vr v$ plus,
intuitively speaking, removing $\mu$ from places $\registers$ and putting 
$\mu'$ to places $\overline \registers$.
In addition, all transitions of the form
\[
\Big((\overline\ell', \nu^{-1}(\bot)), \  \nu \ominus \overline {\nu}, \ (\ell', \nu^{-1}(\bot))\Big)
\]
are added to $T_1$, where $\nu : \registers \to (\A \cup \set{\bot})$ is any register valuation.
Intuitively, these transitions flash back all tokens
from places $\overline \registers$ to the corresponding places $\registers$.
This yields a \dvass $\V_1 := (\Loc_1, \emptyset, \normalcounters, \atomcounters_1,T_1)$
computable from $\V$,
and its location
$\ell_1 = (\ell,\eta^{-1}(\bot))$ 
corresponding to state $q$ 
such that the coverability sets in $\V$ (on the left) is computable from the one in $\V_1$ (on the right):
\begin{claim}
$\coverset {q(\vr v)} \ =  \setof{(\ell', \mu')(\vr v')}{(\ell', (\mu')^{-1}(\bot))(\vr v' \oplus \mu') \in \coverset {\ell_1(\vr v \oplus \eta)}}$.
\end{claim}
%

As the second step, we dispose of locations $\Loc_1$ by moving them to plain places.
We set $\normalcounters_2:= \normalcounters \cup L_1 \cup \widetilde{L_1}$,
where $\widetilde{L_1} = \setof{\widetilde{\ell}}{\ell \in L_1}$,
and transform each transition $t= (\ell, \vr v, \ell') \in T_1$ into a transition in $T_2$:
\[
(\singl, \vr v  \ominus \ell \oplus \widetilde{\ell'}, \singl).
\]
We also add a new transition $(\singl, \ell \ominus \widetilde{\ell}, \singl)$ for every $\ell \in L_1$.
This yields a \dvass $\V_2 := (\set{\singl}, \emptyset, \normalcounters_2, \atomcounters_1, T_2)$
computable from $\V_1$, and the corresponding configuration $\singl(\vr v\oplus \ell)$
such that the coverability sets in $\V_1$ (on the left) is computable from the one in $\V_2$ (on the right):
\begin{claim}
$\coverset {\ell(\vr v)} \ = \ \setof{\ell'(\vr v')}{\vr v' \oplus \ell' \in \coverset {\singl(\vr v \oplus \ell)}}$.
\end{claim}

Eventually, as the last step we get rid of plain places $\normalcounters_2$ by moving them to atom ones, 
and considering atoms residing on these atom places irrelevant.
Let $\atomcounters_3 := \normalcounters_2 \cup \atomcounters_1$.
In order to transfer transitions $T_2$ from plain places to the new atom places, we introduce the projection
mapping
$\pi : (\normalcounters_2\cup \atomcounters_1){\times}\A \to \normalcounters_2\cup (\atomcounters_1{\times}\A)$,
\[
(h,a) \mapsto h \qquad (p,a) \mapsto (p,a) \qquad\qquad  (h\in \normalcounters_2, p\in \atomcounters_1, a\in\A),
\]
that forgets, intuitively speaking, about atoms on the new atom places.
It extends uniquely to a commutative group homomorphism $\pi$ from 
$\cg{(\normalcounters_2\cup \atomcounters_1){\times}\A}$ to $\cg{\normalcounters_2\cup (\atomcounters_1{\times}\A)}$.
We define transitions as the inverse image of $T_2$ along $\pi$:
\[
T_3 \ := \ \pi^{-1}(T_2).
\] 
This yields a \dvas $\V_3 := (\set{\singl}, \emptyset, \emptyset, \atomcounters_3,  T_3)$.
We observe that $\pi^{-1}(\vr v)$ is orbit-finite for every vector $\vr v$, and therefore
$\pi^{-1}(T_2)$, being orbit-finite union of orbit-finite sets, is itself orbit-finite \cite[Ex.~62]{atombook}.
Therefore $\V_3$ is computable from $\V_2$.
The coverability set in $\V_2$ is computable from the one in $\V_3$,
since coverability sets commute the projection
(the coverability set on the left is in $\V_2$, while the one on the right is in $\V_3$):
\begin{claim}
$\coverset {\singl(\pi(\vr w))} \ = \ 
\pi(\coverset {\singl(\vr w)})$.
\end{claim}
Indeed, in order to compute a representation 
$\coverset {\singl(\vr v)} \ = \ \down{g_1} \cup \ldots \cup \down{g_n}$ in $\V_2$, we take
any $\vr w$ with $\pi(\vr w) = \vr v$, compute a representation 
$\coverset {\singl(\vr w)} \ = \ \down{f_1} \cup \ldots \cup \down{f_n}$ in $\V_3$
using \cite[Thm.~3.5]{HLLLST16}, and 
modify the functions $f_i$ by summing up, for every $h\in \normalcounters_2$, namely
(under the proviso that $\omega+n=\omega+\omega=\omega$):
\[
g_i(h) \ := \ \sum_{a\in\A} f_i(h,a) \qquad
g_i(p,a) \ := \ f_i(p,a)
\qquad\qquad  (h\in \normalcounters_2, p\in \atomcounters_1, a\in\A).
\]
This concludes the proof.
\end{proof}

\section{Sufficient condition for \dvass bi-reachability}

In this and in the next section we prove Theorem \ref{thm:bireach}.
Throughout the rest of the paper let
$\V = (\Loc, \registers, \normalcounters, \atomcounters, T)$ be an input \dvass.
Relying on Lemma \ref{lem:wlog} we investigate bi-reachability of 
$q(\vr 0)$ and $q'(\vr 0)$, for states
$q = (\ell, \botval)\in Q$ and $q'=(\ell',\botval')\in Q$ with empty register valuations.
The states $q, q'$ are invariant
under the action of $\Aut \A$, which is crucial in the sequel:
\begin{claim} \label{claim:invq}
    For every $\sigma\in\Aut \A$, we have $\sigma(q) = q$ and $\sigma(q') = q'$.
\end{claim}

We now formulate a sufficient condition for bi-reachability of $q(\vr 0)$ and
$q'(\vr 0)$,
as an adaptation of the classical $\Theta_1$ and $\Theta_2$ conditions.
In the context of bi-reachability, it is enough to rely on a simplified version of these conditions
given by Lemma \ref{lem:suffvass}.
We write below $\vr v \gg \vr 0$ to mean that
$\vr v(h) > 0$ for every $h \in \normalcounters$, and
for every $p \in \atomcounters$ there is some $a\in \A$ such that $\vr v(p, a) > 0$.

\begin{itemize}
    \item$\myTheta_1$: 
    There are pseudo-runs, each of them using some transition from every orbit in $T$:
\begin{align} 
\begin{aligned}
\label{eq:psruns}
    &q(\vect{0}) \pseudorunarrow q'(\vect{0}) \qquad\qquad
    q(\vect{0}) \backpseudorunarrow q'(\vect{0}).
\end{aligned}
\end{align}
    \item$\myTheta_2$: 
    For some vectors $\vr \Delta, \vr \Delta', \vr \Gamma, \vr \Gamma' \gg \vr 0$, there are runs:
\begin{align}
\begin{aligned} \label{eq:runs}
        &q(\vect{0}) \runarrow q(\vr \Delta) \qquad\quad &&
        q'(\vr \Delta') \runarrow q'(\vect{0}) \\
        &q(\vect{0}) \backrunarrow q(\vr \Gamma)  &&
        q'(\vr \Gamma') \backrunarrow q'(\vect{0}).
\end{aligned}
\end{align}
\end{itemize}

\begin{lemma} \label{lem:suff}
$\myTheta_1 \land \myTheta_2$ implies $q(\vect{0}) \runarrow q'(\vect{0})$ and $q(\vect{0}) \backrunarrow q'(\vect{0})$.
\end{lemma}
\begin{proof}
    Assume $\V$ satisfies $\myTheta_1 \land \myTheta_2$.
    Let $S$ be the union of supports of
    the two pseudo-runs \eqref{eq:psruns} and the four runs \eqref{eq:runs}.
    Recall that $\Autparco S \A \subseteq \Aut \A$ denotes the subset of those automorphisms $\sigma$ that
    are identity outside $S$: $\sigma(a) = a$ for every $a\notin S$. 
    When restricted to $S$, each such automorphism is a permutation,
    i.e.,
    $\sigma(S) = S$.
    In the sequel we will apply permutations $\sigma\in \Autparco S \A$ to atoms from $S$ only, and therefore
    the value $\sigma(a)=a$, for $a\notin S$, will be irrelevant.
    The set $\Autparco S \A$ is finite,
    $\size{\Autparco S \A} = \size S !$.

    We define a (plain) \vass $\V_S$ by, intuitively speaking, restricting the set of atoms to the finite set 
    $S$.
    The set of locations of $\V_S$ is $L_S := \Loc \times (\registers\to S \cup \{\bot\})$, 
    its places 
    are $\normalcounters_S := \normalcounters \cup \atomcounters {\times} S$,
    and its transitions $T_S\subseteq T$ are all transitions of $\V$ that use only atoms from $S$.
    Formally, $\V_S = (L_S, \emptyset, \normalcounters_S, \emptyset, T_S)$, where
    $
    T_S := \setof{t\in T}{\supp t \subseteq S}.
    $
    We claim that the \vass satisfies the conditions $\Theta_1$ and $\Theta_2$ of Lemma \ref{lem:suffvass}.
    
\smallskip

    We consider $\Theta_1$ first. 
    Let $\pi : q(\vr 0) \pseudorunarrow q'(\vr 0)$ and $\pi' : q'(\vr 0) \pseudorunarrow q(\vr 0)$ 
    be the pseudo-runs in \eqref{eq:psruns}.
    By applying all permutations $\sigma\in\Autparco S \A$ to their concatenation $\pi; \pi' :
    q(\vr 0) \pseudorunarrow q'(\vr 0)  \pseudorunarrow q(\vr 0)$ , and concatenating all the $\size S !$ 
    resulting cyclic pseudo-runs, we get a cyclic pseudo-run
    $\delta : q(\vr 0) \pseudorunarrow q(\vr 0)$.
    This pseudo-run uses every transition from $T_S$ at least once, since $\pi$ uses a representative of every
    orbit of $T$, and the following fact holds:
    \begin{claim} \label{claim:v}
    Let $t, t'\in T$ be two transitions in the same orbit such that 
    $\supp {t}, \supp {t'}\subseteq S$.
    Then $t' = \sigma(t)$ for some $\sigma \in \Autparco S \A$.
    \end{claim}
    Likewise we get a cyclic pseudo-run
    $\delta' : q'(\vr 0) \pseudorunarrow q'(\vr 0)$ that uses every transition from $T_S$ at least once. 
    Furthermore, for every $m\in\N$, the $m$-fold concatenation of $\delta$ or $\delta'$ yields a cyclic pseudo-run that
    uses every transition from $T_S$ at least $m$ times. 
    We thus have two pseudo-runs 
\begin{align*} 
    \delta^m; \pi \ : \  
    q(\vect{0}) \pseudorunarrow q'(\vect{0}) 
    \qquad\qquad\qquad
    (\delta')^m; \pi' \ : \  
    q(\vect{0}) \backpseudorunarrow q'(\vect{0})
\end{align*}
   each of them using every transition from $T_S$ at least $m$ times.  Thus the \vass $\V_S$ satisfies 
   two instances of 
   $\Theta_1$, one towards a run $q(\vect{0}) \runarrow q'(\vr 0)$ and
   the other one towards 
   a run $q'(\vect{0}) \runarrow q(\vr 0)$.

\smallskip

    Now we concentrate on $\Theta_2$.
    We proceed similarly as before, namely apply all automorphisms 
    $\sigma \in \Autparco S \A$ to $\vr \Delta$, and sum up all the resulting vectors:
    \[
    \vr \Delta_S \ := \ \oplus \setof{\sigma(\vr \Delta)}{\sigma \in \Autparco S \A}.
    \]
    Let $\overline {\vr \Delta}_S : \cgN{\normalcounters \cup \atomcounters {\times}S}$ be the restriction of 
    $\vr\Delta_S$ to $\normalcounters \cup \atomcounters{\times S}$.
    Knowing that $\vr\Delta \gg \vr 0$, we deduce that
    $\vr \Delta_S(h)>0$ for every $h \in \normalcounters$, and
    $\vr\Delta_S(p, a) > 0$ for every $(p, a) \in \atomcounters \times S$.
    In other words, $\overline{\vr\Delta}_S \gg \vr 0$.
    By applying all automorphisms $\sigma \in \Autparco S \A$ to the run $q(\vect{0}) \runarrow q(\vr \Delta)$
    in $\myTheta_2$, and concatenating
    all the resulting runs, we get a run 
    $q(\vect{0}) \runarrow q(\vr\Delta_S)$ in $\V$. 
    Clearly, only atoms from $S$ appear in this run, and therefore it is also a run
    $q(\vect{0}) \runarrow q(\overline{\vr\Delta}_S)$ in $\V_S$.
    In a similar way we define vectors $\overline{\vr\Delta}_S'$, $\overline{\vr\Gamma}_S$ and $\overline{\vr\Gamma}_S'$, 
    and the corresponding runs in $\V_S$:
\begin{align}
\label{eq:runsS1}
        &q(\vect{0}) \runarrow q(\overline{\vr\Delta}_S) \qquad\qquad
        q'(\overline{\vr \Delta}'_S) \runarrow q'(\vect{0}) \\
\label{eq:runsS2}
        &q(\vect{0}) \backrunarrow q(\overline{\vr\Gamma}_S)  \qquad\qquad
        q'(\overline{\vr\Gamma}'_S) \backrunarrow q'(\vect{0}).
\end{align}
    Therefore, the \vass $\V_S$ satisfies two instances of $\Theta_2$,
    one towards a run $q(\vect{0}) \runarrow q'(\vr 0)$ and the other one
    towards a run $q'(\vect{0}) \runarrow q(\vr 0)$.

\smallskip

   Finally, using Lemma \ref{lem:suffvass}  we deduce two runs in $\V_S$, which
   are automatically also runs in $\V$. This completes the proof.
\end{proof}

\section{Reduction algorithm}

As the \emph{rank} of a \dvass $\V = (\Loc, \registers, \normalcounters, \atomcounters, T)$ 
we take the triple $\rank \V = (\size {\atomcounters}, \size {\normalcounters}, \orbsize T)$,
consisting of the number of atom places, the number of plain places, 
and the number of orbits $\orbsize T$ the set $T$ partitions into.
Ranks are compared lexicographically.

Given a \dvass $\V$, the algorithm verifies the conditions $\myTheta_1$ and $\myTheta_2$.
If they are all satisfied, it answers positively, relying on Lemma \ref{lem:suff}.
Otherwise, depending on which of the conditions is violated, the algorithm either immediately answers negatively,
or applies a reduction step, as outlined below in Sections \ref{sec:Theta1} and \ref{sec:Theta2}.
Each of the steps produces a new \dvass $\widehat \V$ of strictly smaller rank, which guarantees termination.
Finally, when both $\atomcounters$ and $\normalcounters$ are empty,
the problem reduces to reachability from $q$ to $q'$ in state graph $\csg \V$, which 
is decidable due to:
\begin{lemma} \label{lem:ofg}
For a set $E\subseteq Q\times Q$ of edges between states, given as a finite union of orbits,
and a pair $(s, s')\in Q\times Q$,
it is decidable if there is a path from $s$ to $s'$ in the graph $(Q, E)$.
\end{lemma}
\begin{proof}
The orbit of an edge $((\ell,\nu), (\ell',\nu'))\in E$ is determined by the following data: 
locations $\ell, \ell'$;
the inverse images 
$
\nu^{-1}(\bot), \ 
(\nu')^{-1}(\bot);
$
and the equality type of the remaining entries in $\nu$ and $\nu'$, that is:
\[
\setof{(r,r')}{\nu(r)=\nu(r')\neq\bot} \quad
\setof{(r,r')}{\nu'(r)=\nu'(r')\neq\bot} \quad
\setof{(r,r')}{\nu(r)=\nu'(r')\neq\bot}.
\]
Using equational reasoning, one computes the transitive closure $E^*$ of $E$, by consecutively
adding to $E^*$ every new orbit which is forced to be included in $E^*$ 
by some two orbits already included in $E^*$, until saturation.
Termination is guaranteed as $Q\times Q$ is orbit-finite.
The transitive closure is thus forcedly a finite union of orbits.
Finally, one tests if the orbit of $(s,s')$ is included in $E^*$. 
\end{proof}

For future use we note an immediate consequence of the above proof:
for every pair of states, if there is a path from one to the other, then there is also a path of bounded length.
This implies a bound on the number of atoms involved:
\begin{corollary} \label{cor:ofg}
There is an effective bound $b(Q)\in\N$ such that whenever there is a path from $s\in Q$ to $s'\in Q$ 
in the graph $(Q, E)$, 
then there is such a path $\pi$ with $\size{\supp \pi}\leq b(Q)$.
\end{corollary}

Below we describe the two reduction steps, proving their progress property (decreasing rank), 
correctness and effectiveness.

\subsection{Violation of $\myTheta_1$} \label{sec:Theta1}

Suppose $\V$ violates $\myTheta_1$. 
If states $q,q'$ are not in the same strongly connected component of $\csg T$,
which is testable using Lemma \ref{lem:ofg},
the configurations $q(\vr 0)$, $q'(\vr 0)$ are clearly not bi-reachable and the algorithm answers negatively.
Otherwise,
the algorithm constructs a \dvass $\widehat \V$ of smaller rank, as defined below, 
such that bi-reachability of $q(\vr 0)$ and $q'(\vr 0)$
in $\V$ is equivalent to their bi-reachability in $\widehat \V$.

Let $\csg T = (Q, E)$.
Transitions witnessing bi-reachability of $q(\vr 0)$ and $q'(\vr 0)$, namely used in some cyclic run
%
$q(\vect{0}) \runarrow q'(\vect 0) \runarrow q(\vect 0)$,
%
form a cycle in $\csg T$.
As a consequence, a transition $(s, \vr v, s')\in T$ may be useful for bi-reachability 
only if the edge $(s, s')$ belongs to the strongly connected component of $\csg T$ containing $q$ and $q'$.
Therefore, bi-reachability of $q(\vr 0)$, $q'(\vr 0)$ in $\V$ reduces to bi-reachability of $q(\vr 0)$, $q'(\vr 0)$
in $\widehat \V$ obtained by restriction to the strongly connected component of $q$ and $q'$.
This component is computed by enumerating all orbits included in $E$.
For every orbit $o\subseteq E$ one chooses a representative $(s, s')\in o$, 
and uses Lemma \ref{lem:ofg} to test reachability, in $\csg T$, for  
the four pairs: $(q, s)$, $(s', q)$, $(q', s)$ or $(s', q')$.
Then one removes from $T$ all orbits of transitions $(s, \vr v, s')$ such that
reachability test fails for any of the four pairs above. 
The resulting set of transitions is still a finite union of orbits.
Consequently, from now on we may assume, w.l.o.g., that 
$\csg T$ is strongly connected (we ignore isolated vertices).

\para{Useful transitions}

By a finite multiset of transitions we mean a nonnegative vector $\vr f : \cgN T$. 
Given such a finite multiset, 
let $\cond q {q'}$ denote conjunction of the following conditions:

\begin{enumerate}
    \item[(a)]
    the sum of effects of all transitions in the multiset is $\vr 0$,
    \item[(b)]
    for every state $s\notin\set{q,q'}$, 
    the number of transitions incoming to $s$ equals 
    the number of ones outgoing from $s$,
    \item[(c)]
    the number of transitions outgoing from $q$ exceeds by one 
    the number of incoming ones,
    \item[(d)]
    the number of transitions incoming to $q'$ exceeds by one 
    the number of outgoing ones;
\end{enumerate}

    \noindent
    Symmetrically, let $\cond {q'} q$ denote the conjunction of (a), (b) and the symmetric versions of (c) and (d)    
    with
    $q$ and $q'$ swapped.
Let $O = \setof{\orbitof t}{t \in T}$ be the set of all orbits in $T$.
We call an orbit $o\in O$ \emph{useful} if there are two finite multisets of transitions $\vr f, \vr f'$ 
satisfying $\cond q {q'}$ and $\cond {q'} {q}$, respectively, each of them containing some transition from $o$.

\begin{lemma} \label{lem:theta1}
$\myTheta_1$ holds if and only if all orbits of transitions are useful.
\end{lemma}
\begin{proof}
    The only if direction of the characterisation is immediate, as the multiset of transitions used in a~pseudo-run
    $q(\vr 0) \pseudorunarrow q'(\vr 0)$
    necessarily contains some transition from every orbit and satisfies all the conditions (a)--(d),
    and likewise for a~pseudo-run $q'(\vr 0) \pseudorunarrow q(\vr 0)$.
    For the opposite direction, suppose that for every orbit $o\in O$ there are finite multisets $\vr f_o, \vr f'_o$ 
    satisfying $\cond q {q'}$
    and $\cond {q'} q$, respectively, each of them containing some transition from $o$.
    Let 
\[
\vr f \ := \ 
\oplus \setof{\vr f_o}{o \in O} \qquad\qquad
\vr f' \ := \ 
\oplus \setof{\vr f'_o}{o \in O}
\qquad\qquad
    S  \ := \ \supp {\vr f} \cup \supp {\vr f'},
\]
    where $\oplus$ denotes the multiset sum operator.
    Thus $S$ is the (finite) set of atoms used in all the transitions appearing in $\vr f$ or $\vr f'$.
    Similarly as in the proof of Lemma \ref{lem:suff} we use the subgroup
    $\Autparco S \A \subseteq \Aut \A$ of automorphisms $\sigma$ of $\A$ that
    are identity outside $S$, and
    define a plain \vass $\V_S = (L_S, \emptyset, \normalcounters_S, \emptyset, T_S)$
    by restricting the set of atoms to $S$.
    Locations of $\V_S$ are $\Loc_S := \Loc \times (\registers \to S \cup \{\bot\})$, its places are 
    $\normalcounters_S := \normalcounters \cup \atomcounters{\times} S$,
    and its transitions $T_S\subseteq T$ are all transitions of $\V$ that use only atoms from $S$:
    $T_S = \setof{t\in T}{\supp t\subseteq S}.$
    The state graph $\csg{T_S}$ is a subgraph of $\csg T$.
    As $\csg T$ is strongly connected,
    we use Corollary \ref{cor:ofg} to deduce that, for sufficiently large $S$, the state graph $\csg{T_S}$
    is also strongly connected.
    Therefore we enlarge $S$, if necessary, to assure that $\csg{T_S}$ is strongly connected.
    
    As finite multisets $\vr f, \vr f'$ are just nonnegative vectors $\cgN{T}$, 
    they inherit the natural (pointwise) action of $\Aut \A$. 
    Basing on $\vr f$, $\vr f'$ we define two larger multisets of transitions
    by applying all automorphisms from $\Autparco S \A$ to $\vr f$ and $\vr f'$, respectively, and summing up
    all the resulting multisets:
    \[
        \vr g \ := \ \oplus \setof{\sigma(\vr f)}{\sigma \in \Autparco S \A}    \qquad\qquad
        \vr g' \ := \ \oplus \setof{\sigma(\vr f')}{\sigma \in \Autparco S \A}. 
    \]
    By Claim \ref{claim:v}, 
    each of $\vr g, \vr g'$ contains all transitions from $T_S$.
    Furthermore, the multiset $\vr h = \vr f \oplus \vr g \oplus \vr g'$ satisfies 
    $\cond {\widehat q} {\widehat q'}$,
    where
    $\widehat q$, $\widehat q'$ are locations (=states) of $\V_S$ corresponding to $q$ and $q'$ respectively.
    Likewise, the multiset $\vr h' = \vr f' \oplus \vr g \oplus \vr g'$ satisfies 
    $\cond {\widehat q'} {\widehat q}$.
    Using the standard Euler argument in the (strongly) connected graph $\csg{T_S}$, 
    and relying on conditions (b)--(d),
    we deduce existence of a pseudo-run in $\V_S$
    that uses exactly transitions
    $\vr h$.
    Due to condition (a), this is a pseudo-run $\widehat q(\vr 0) \pseudorunarrow \widehat q'(\vr 0)$
    in $\V_S$.
    Likewise we deduce a pseudo-run $\widehat q'(\vr 0) \pseudorunarrow \widehat q(\vr 0)$
    in $\V_S$.
    The pseudo-runs are essentially also pseudo-runs in $\V$, both supported by $S$,
    and both using some transition from every orbit in $T$.
    This completes the proof of the characterisation. 
\end{proof}

\para{Reduction step}
We define a~\dvass $\widehat \V$ by removing some useless orbit of transitions,
$\widehat \V = (\Loc, \registers, \normalcounters, \atomcounters, \widehat T)$.
It has the same locations, registers and places as $\V$.

\begin{lemma}[Progress]
$\orbsize{\widehat T} < \orbsize{T}$, and hence $\rank{\widehat \V} < \rank{\V}$.
\end{lemma}
\begin{proof}
As  some useless orbit of transitions is removed from $\widehat T$,
we have $\orbsize{\widehat T} < \orbsize{T}$, and therefore
$\rank {\widehat \V} = (\size {\atomcounters}, \size {\normalcounters}, \orbsize {\widehat T}) < 
(\size {\atomcounters}, \size {\normalcounters}, \orbsize T) = \rank \V$.
\end{proof}

\begin{lemma}[Correctness]
    The configurations 
    $q(\vect{0})$, $q'(\vect{0})$ are bi-reachable
    in $\V$ if and only if they are bi-reachable
    in $\widehat \V$.
\end{lemma}
\begin{proof}
Indeed, useless transitions can not be used in runs between $q(\vect{0})$ and $q'(\vect{0})$.
\end{proof}

\begin{lemma}[Effectiveness] \label{lem:eff1}
    The condition $\myTheta_1$ is decidable.
    When it fails, some useless orbit of transitions is computable.
\end{lemma}
\begin{proof}
    It is sufficient to prove that it is decidable if a given orbit $o\in O$ is useful.
    We show decidability by reduction to {\sc Multiset sum} (recall Lemma \ref{lem:HLT17}), i.e., we use
    the algorithm for {\sc Multiset Sum} to 
    check existence of a finite multiset $\vr f$ that satisfies conditions (a)--(d) and contains a transition from $o$
    (and likewise to check existence of $\vr f'$).
    Let $X = \normalcounters \cup \atomcounters {\times} \A$ and     $Y = X \cup Q \cup \set{*}$.
    The set $Y$ is orbit-finite.
    Let every transition $t = (s, \vr v, s')\in T$ determine a vector $\vr y_t : \cg{Y}$, and let every transition
    $t= (s, \vr v, s')\in o$ determine additionally a vector $\vr x_t : \cg{Y}$, each of them extending
    the vector $\vr v : \cg X$ as follows:
    \begin{align} \label{eq:xy}
    \begin{aligned} 
    &\vr x_t(s) = -1 \qquad\qquad
    \vr x_t(s') = 1 \qquad\qquad    
    \vr x_t(*) = 1, \\
    &\vr y_t(s) = -1 \qquad\qquad
    \vr y_t(s') = 1 \qquad\qquad    
    \vr y_t(*) = 0,
    \end{aligned}
    \end{align}
    and $\vr x_t(r) = \vr y_t(r) = 0$ for other states $r\in Q\setminus\set{s,s'}$.
    Intuitively,     
    the first two columns track contribution of $t$ to the number of transitions
    outgoing from $s$, or incoming to $s'$, 
    while the latter column tracks the number of usages of transitions from $o$.
    Likewise, let the target vector $\vr b : \cg {Y}$ extend $\vr 0 : \cg X$ by
    \begin{align} \label{eq:b}
    \begin{aligned} 
    \vr b(q) = -1 \qquad\qquad
    \vr b(q') = 1 \qquad\qquad
    \vr b(*) = 1,
    \end{aligned}
    \end{align}
    and $\vr b(s) = 0$ for other states $s\in Q \setminus \set{q,q'}$. 
    Let $M = \setof{\vr y_t}{t\in T} \cup \setof{\vr x_t}{t\in o}$, and consider the instance
    let $(M, \vr b)$ of {\sc Multiset Sum}.
    Observe that solutions of $(M, \vr b)$ necessarily use exactly one vector $\vr x_t$ exactly once, 
    as in \eqref{eq:btu} below. 
    It remains to prove that some finite multiset   $\vr f = \set{t, u_1, \ldots, u_m}$ of transitions from $T$, 
    where $t\in o$,
    satisfies the conditions (a)--(d)  exactly when
    %
    \begin{align} \label{eq:btu}
    \vr b \ = \ \vr x_{t} + \vr y_{u_1} + \ldots + \vr y_{u_m}.
    \end{align}
    %
    As $\vr b(c) = 0$ for all $c\in \normalcounters \cup \atomcounters{\times}\A$, condition (a) is equivalent to
    the equality \eqref{eq:btu} restricted to $\normalcounters \cup \atomcounters{\times}\A$.
    By the first two columns in \eqref{eq:xy} and \eqref{eq:b}, and since $\vr b(s) = 0$ for $s\in Q\setminus\set{q, q'}$,
    condition (b) is equivalent to
    the equality \eqref{eq:btu} restricted to $Q\setminus\set{q,q'}$, and
    the conditions (c) and (d) are equivalent to
    the equality \eqref{eq:btu} restricted to $\set{q,q'}$.
\end{proof}

\subsection{Violation of $\myTheta_2$} \label{sec:Theta2}

Suppose $\V$ violates $\myTheta_2$.
We  define a \dvass $\widehat \V$ of smaller rank and two its states $\widehat q, \widehat q'$  
such that bi-reachability of $q(\vr 0)$ and $q'(\vr 0)$ in $\V$ is equivalent to bi-reachability
of $\widehat q(\vr 0)$ and $\widehat q'(\vr 0)$ in $\widehat \V$.

\para{Pumpability and boundedness}
Any run of the form $q(\vr 0) \runarrow q(\vr \Delta)$ (resp.~$q(\vr \Delta) \runarrow q(\vr 0)$) 
we call \emph{forward} (resp.~\emph{backward}) \emph{pump} from $q$.
Likewise we define forward (resp.~backward) pumps from $q'$.

A plain place $h\in \normalcounters$ (resp.~an atom place $p\in \atomcounters$) we call \emph{forward pumpable} from $q$ if there is a pump
$q(\vr 0) \runarrow q(\vr \Delta)$ such that $\vr \Delta(h) > 0$
(resp.~$\vr \Delta(p,a) > 0$ for some atom $a\in \A$).
Symmetrically, 
a place $h\in \normalcounters$ (resp.~$p\in \atomcounters$) we call 
\emph{backward pumpable} from $q$ if there is a pump
$q(\vr \Delta) \runarrow q(\vr 0)$ such that $\vr \Delta(h) > 0$
(resp.~$\vr \Delta(p,a) > 0$ for some atom $a\in \A$).
Likewise we define places forward (resp.~backward) pumpable from $q'$.
Finally, a (plain or atom) 
place is called \emph{pumpable} if it is forward and backward pumpable both from $q$ and
from $q'$.
Otherwise, the place is called \emph{unpumpable}.

\begin{lemma} \label{lem:theta2}
$\myTheta_2$ holds if and only if all places are pumpable.
\end{lemma}
\begin{proof}
In one direction, 
$\myTheta_2$ amounts to \emph{simultaneous} (= using one pump) 
forward and backward pumpability of all places, from both $q$ and $q'$.
In the converse direction, suppose all places are forward and backward pumpable
from both $q$ and $q'$.
We observe that
forward (resp.~backward) pumps from $q$ compose, and hence all places are \emph{simultaneously} 
forward (resp.~backward) pumpable from $q$. Likewise for $q'$.
\end{proof}

We now introduce a suitable version of boundedness.
In case of plain places $h\in \normalcounters$ boundedness applies, as expected, to the value $\vr v(h)$
in a configuration $s(\vr v)$. 
On the other hand in case of atom places $p\in \atomcounters$, boundedness applies to the number of tokens on a~place.
For uniformity we define the~size of a~place $p\in\atomcounters$ in a~vector $\vr v$ as
\begin{align} \label{eq:size}
\vr v(p) = \sum \setof{\vr v(p,a)}{a\in\A, \ \vr v(p,a)>0} \in \N.
\end{align}
For a family $\mathcal F$ of runs, 
we say that a place $c\in \normalcounters \cup \atomcounters$ is \emph{bounded on $\mathcal F$ by $B\in\N$}
if for every configuration $s(\vr v)$ appearing in every run in the family $\mathcal F$,
$\vr v(c) \leq B$.
A place is \emph{bounded on}  $\mathcal F$, if it is bounded on $\mathcal F$ by some $B$.
Otherwise, the place is called \emph{unbounded} on $\mathcal F$.

As forward (resp.~backward) pumps compose,
every place which is forward (resp.~backward) pumpable from $q$ (resp.~$q'$) is necessarily unbounded
on respective pumps.
We show the opposite implication, namely unpumpable places are bounded on respective pumps:

\begin{lemma} \label{lem:bounded}
    A place which is not forward (resp.~backward) pumpable from $q$ is bounded on 
    forward (resp.~backward) pumps from $q$. Likewise for $q'$.
\end{lemma}
\begin{proof}
    W.l.o.g.~we focus on forward pumps from $q$ only.
    (Backward pumps are tackled similarly as forward ones, but in the \emph{reversed} \dvass,
    whose transitions $\setof{(s', -\vr v, s)}{(s,\vr v, s')\in T}$ are inverses of transitions of $\V$.)
    Assuming that a place $c\in \normalcounters \cup \atomcounters$ is unbounded on
    forward pumps from $q$, we 
    show that $c$ is forward pumpable from $q$.
    Unboundedness of $c$ means that for every $i\in\N$ there is a run 
    \begin{align} \label{eq:vm}
    q(\vr 0) \runarrow s_i(\vr v_i) \runarrow q(\vr \Delta_i)
    \end{align}
    such that  
    $n_i := \vr v_i(c) > i$.
    By choosing a subsequence of $(n_i)_i$ we may assume w.l.o.g.~that this sequence is strictly increasing.
    As $\wqorel$ is a \wqo, for some $j<i$ we have $s_j (\vr v_j) \wqorel s_i(\vr v_i)$, i.e.,
    there is some $\sigma\in\Aut\A$ such that  $\sigma(s_j) = s_i$ and
    $\sigma(\vr v_j) \leq \vr v_i$.
    Equivalently, $\sigma(\vr v_j) + \vr \Delta = \vr v_i$ for some nonnegative vector $\vr \Delta$.
    Recall that $q = (\ell, \botval)$ and hence $\sigma(q) = q$.
    Therefore, we can construct a new run, by first using the first half of
    \eqref{eq:vm} to go from $p(\vr 0)$ to $s_i(\vr v_i)$, and then applying $\sigma$ to 
    the second half $s_j(\vr v_j) \runarrow q(\vr \Delta_j)$ to return from 
    $\sigma(s_j(\vr v_j)) = \sigma(s_j)(\sigma(\vr v_j))$:
    \[
    q(\vr 0) \runarrow s_i(\vr v_i) = \sigma(s_j)(\sigma(\vr v_j) + \vr \Delta)  \runarrow q(\sigma(\vr \Delta_j) + \vr \Delta).
    \]
    By strict monotonicity of $(n_i)_i$ we have $n_j < n_i$, i.e., $\vr v_j(c) < \vr v_i(c)$, which implies 
    $\vr \Delta(c) > 0$.
    The place $c$ is thus forward pumpable from $q$, as required.
\end{proof}

\para{Reduction step}
Let $B\in\N$ be the universal bound for all unpumpable places on all respective pumps.
Formally, let's assume that $B$ bounds every place which is not forward (resp.~backward) pumpable 
from $q$, on all forward (resp.~backward) pumps from $q$, and
$B$ also bounds every place which is not forward (resp.~backward) pumpable 
from $q'$, on all forward (resp.~backward) pumps from $q'$.
As $\V$ violates $\myTheta_2$, we know by Lemma \ref{lem:theta2} that there are some unpumpable
places.
We define a \dvass $\widehat \V = (\widehat \Loc, \widehat \registers, \widehat \normalcounters, \widehat \atomcounters, \widehat T)$ 
by, intuitively speaking, removing some such place.
Let $X = \normalcounters \cup \atomcounters{\times}\A$.

\medskip
{\bf Case 1: some plain place is unpumpable.}
Let $\widehat \atomcounters := \atomcounters$ and $\widehat \registers = \registers$.
We choose an arbitrary unpumpable $h\in \normalcounters$, 
and let $\widehat \normalcounters := \normalcounters\setminus\set{h}$.
We keep track of values on $h$ by extending the locations $\widehat \Loc := \Loc \times \setto{B}$.
Finally, transitions $\widehat T$ are obtained from transitions $T$ by replacing each 
$
t \ = \ ((\ell, \mu), \vr v, (\ell', \mu')) \in T
$
by all the transitions of the form
\[
\widehat t \ =  \ ((\pair \ell n, \mu), \ \vr{w}, (\pair {\ell'}{n+\vr v(h)}, \mu')),
\]
where 
$\vr{w}$ is the restriction of $\vr v : \datavec$ to 
$\widehat \normalcounters \cup \atomcounters{\times}\A$.
Let
$\widehat q = (\pair \ell 0, \botval)$ and $\widehat q'=(\pair {\ell'}{0}, \botval)$.

\medskip

{\bf Case 2: some atom place is unpumpable.}
Let $\widehat \Loc := \Loc$ and $\widehat \normalcounters := \normalcounters$.
We choose an arbitrary unpumpable $p\in \atomcounters$, 
remove place $p$, namely $\widehat \atomcounters := \atomcounters\setminus\set{p}$, and
add $B$ new registers, namely $\widehat \registers = \registers \cup \{r_1, \dots, r_{B}\}$, which store,
intuitively speaking, every possible content of the place $p$.
Using the new registers
the transitions $\widehat T$ track, intuitively speaking, 
the effect of transitions from $T$ on the removed place $p$.
Every register valuation $\mu$ naturally induces a vector $\widehat \mu : \cgN{\set{p}{\times}\A}$,
namely
$
\widehat \mu \ := \ \oplus \setof{(p,a)}{\mu(r_i)=a\neq\bot, \ i \in \setfromto 1 B}.
$
Using this notation we define the transitions $\widehat T$ by replacing every transition 
$((\ell, \nu), \vr v, (\ell', \nu')) \in T$ with all the transitions of the form
\[
\big((\ell, \nu\regplus \mu), \ \vr w, \ (\ell', \nu'\regplus \mu')\big),
\]
where $\vr w$ is the restriction of 
$\vr v : \datavec$ to 
$\normalcounters \cup \widehat \atomcounters{\times}\A$, and
$\vr v  = \vr w \oplus \widehat \mu' \ominus \widehat \mu$.
Let $\widehat q, \widehat q'$ be extensions of the states $q, q'$, respectively, by
empty valuation of all the new registers $\{r_1, \dots, r_{B}\}$.

\begin{lemma}[Progress]
$\rank{\widehat \V} < \rank{\V}$.
\end{lemma}
\begin{proof}
    In the former case $\widehat \atomcounters = \atomcounters$ and $\size {\widehat \normalcounters} < \size \normalcounters$, while in the latter case
    $\size{\widehat \atomcounters} < \size \atomcounters$. In each case, 
    $\rank {\widehat \V} = (\size {\widehat \atomcounters}, \size {\widehat \normalcounters}, \orbsize {\widehat T}) < 
    (\size {\atomcounters}, \size {\normalcounters}, \orbsize T) = \rank \V$.
\end{proof}

\begin{lemma}[Correctness]
    The configurations 
    $q(\vect{0})$, $q'(\vect{0})$ are bi-reachable
    in $\V$ if and only if $\widehat q(\vect{0})$, $\widehat q'(\vect{0})$  are bi-reachable
    in $\widehat \V$.
\end{lemma}
\begin{proof}
    As $\widehat \V$ is obtained from $\V$ by restricting values in configurations on some places,
    each run in $\widehat \V$ is automatically a run in $\V$, and bi-reachability in $\widehat \V$ implies
    bi-reachability in $\V$.
    
    For the converse implication, suppose  $q(\vect{0})$, $q'(\vect{0})$ are bi-reachable
    in $\mathcal{V}$, and fix two arbitrary runs $\pi : q(\vect{0}) \runarrow q'(\vect{0})$ and
    $\pi' : q'(\vect{0}) \runarrow q(\vect{0})$.
    Consider any unpumpable place $c \in \normalcounters \cup \atomcounters$ and suppose, 
    w.l.o.g., that the place is not forward pumpable from $q$.
    By Lemma \ref{lem:bounded} the place is bounded by $B$ on all forward pumps from $q$.
    In particular, it is bounded by $B$ on the composed run $\pi; \pi' : q(\vect{0}) \runarrow q(\vect{0})$,
    i.e., on both $\pi$ and $\pi'$.
    Therefore 
    we know that in each configuration $(s, \vr v)$ in both runs,  $\vr v(c) \leq B$.
    If $c\in \normalcounters$ we keep track of $\vr v(c)$ in locations; and
    if $c\in \atomcounters$,
    we keep track of all tokens on $c$ by storing their atoms in the new registers $\{r_1, \dots, r_B\}$.
    Therefore, both runs are realisable in $\widehat \V$. 
\end{proof}

%

\begin{lemma}[Effectiveness] \label{lem:eff}
The set of pumpable places and the universal bound $B$ are computable
  (and therefore the condition $\myTheta_2$ is decidable).
\end{lemma}
\begin{proof}
For a simple $\omega$-configuration $s(f)$ and an atom place $p\in \atomcounters$ we define
$f(p) = \omega$ if $f(p,a)=\omega$ for some atom $a\in\A$, and otherwise
$
f(p) \ := \ \sum \setof{f(p,a)}{f(p,a)>0, \ a\in\A}.
$
%
Relying on Lemma \ref{lem:theta2}, we test forward (resp.~backward) 
pumpability of every place from $q$ (and likewise from $q'$).
Specifically, to test if a place is forward pumpable from $q$, 
we apply Lemma \ref{lem:cover} and compute a simple-ideal representation of the coverability set
$\coverset {q(\vr 0)}$, given by simple $\omega$-configurations $G = \set{s_1(g_1), \ldots, s_n(g_n)}$.
A plain or atom place $c \in \normalcounters\cup\atomcounters$ is forward pumpable from $q$ if 
for some $i\in\setfromto 1 n$ we have $s_i = q$ and $g_i(c)>0$.
Likewise we test if a place is backward pumpable from $q$
(using the reverse \dvass), or forward (resp.~backward) pumpable from $q'$.
The set of pumpable places is thus computable.

Suppose there is some unpumpable place $c$, w.l.o.g.~say not forward pumpable from $q$.
By Lemma \ref{lem:bounded}, the place $c$ is bounded on all forward pumps from $q$, but may be a priori 
unbounded on some other run, and therefore
it may happen that $g_i(c)=\omega$ for some $i\in\setfromto 1 n$.
Nevertheless we claim the following bound on forward pumps:
\begin{align} \label{eq:B}
B \ := \ \max \setof{g_i(c)}{g_i(c) < \omega, \ i \in \setfromto 1 n}.
\end{align}
\vspace{-5mm}
%
\begin{claim}
The place $c$ is bounded on all forward pumps from $q$ by $B$.
\end{claim}
\begin{claimproof}
Consider an atom place $p\in P$ (the argument for plain places is similar but simpler). 
Towards contradiction, suppose that some configuration $s(\vr v)$ on some forward pump from $q$ 
satisfies $\vr v(p) > B$.
We have thus two runs:
\[
\pi \ : \ q(\vr 0) \runarrow s(\vr v) \qquad\qquad \rho : s(\vr v) \runarrow q(\vr w)
\]
for some $s \in Q$ and nonnegative vector $\vr w$.
Therefore $s(\vr v) \in \coverset {q(\vr 0)}$, and hence
$s(\vr v) \in I = \down{s_j(g_j)}$ for some $j\in\setfromto 1 n$.
As $\vr v(p)$ is larger than all $g_i(p)  < \omega$ for  $i\in\setfromto 1 n$,
we deduce that $g_j(p) = \omega$.
Therefore the ideal $I$ contains configurations $s'(\vr v')$ with arbitrary large values of $\vr v'(p)$.
In particular, $I$ contains some configuration $\vr s'(\vr v')$ with $\vr v(p) < \vr v'(p)$.
As $I$ is directed, 
it contains a configuration $s''(\vr v'')$ such that
\[
s(\vr v) \wqorel s''(\vr v'') \qquad \qquad s'(\vr v') \wqorel s''(\vr v''),
\]
which implies $\vr v(p) < \vr v''(p)$.
Finally, as $I\subseteq \coverset {q(\vr 0)}$, it must contain
a configuration ${\overline s}(\overline{\vr v})$ reachable from $q(\vr 0)$, 
such that
$s''(\vr v'') \wqorel \overline s(\overline{\vr v})$. 
We thus have $s(\vr v) \wqorel \overline s(\overline{\vr v})$ and $\vr v(p) < \overline{\vr v}(p)$, which
means that
for some automorphism $\sigma\in\Aut A$ and vector $\overline{\vr w}$ we have
\[
\sigma(s) = \overline s
\qquad\qquad
\sigma(\vr v) + \overline {\vr w} = \overline{\vr v}
\qquad\qquad
\overline {\vr w}(p)  > 0.
\]
By composing runs $\overline\pi : q(\vr 0) \to \overline s(\overline{\vr v})$ and 
$\sigma(\rho) : \overline s(\sigma(\vr v)) \runarrow q(\sigma(\vr w))$ we get a run
\[
\overline\pi; \sigma(\rho) \ : \ q(\vr 0) \runarrow q(\sigma(\vr w) + \overline{\vr w}),
\]
i.e., we deduce that $p$ is forward pumpable from $q$, which is a contradiction.
\end{claimproof}
The universal bound is computed as the maximum of \eqref{eq:B} for forward and backward pumps 
from $q$ and $q'$, for all unbounded places $c$.
\end{proof}

\section{Final remarks}

We show decidability of the bi-reachability problem for Petri nets with equality data.
The problem subsumes coverability, and reachability in the reversible subclass, and therefore
the result pushes further the decidability border towards the reachability problem.
The latter problem (which we believe to be decidable) is still beyond our reach, and
development of this paper is not sufficient.
For instance, the approach of proving of Lemma \ref{lem:suff} would fail for reachability, as we rely on the fact that 
bi-reachability implies a cycle.
Moreover, $\myTheta_1$ reduction step would fail as well, as it assumes that
a transition (orbit) is either unusable, or usable unboundedly, 
while in case of reachability a transition can be usable only \emph{boundedly}.

Our approach is specific to equality data, and thus we leave unresolved the status of bi-reachability 
in case of ordered data.
In case of ordered data domain the approach of proving Lemma \ref{lem:suff} would fail again,
as the trick  of applying all permutations of $S$ would be impossible. 
Moreover, it is not clear how to implement $\myTheta_2$ reduction step, as no procedure computing 
coverability sets is known.


\bibliography{bib,pn-bib}

\end{document}